\documentclass[twoside,11pt]{article}

\usepackage{jmlr2e}

\jmlrheading{1}{2000}{1-10}{4/00}{10/00}{meila00a}{X. Yu}

\ShortHeadings{}{}
\firstpageno{1}

\usepackage[utf8]{inputenc} % allow utf-8 input
\usepackage[T1]{fontenc}    % use 8-bit T1 fonts
\usepackage{hyperref}       % hyperlinks
\usepackage{url}            % simple URL typesetting
\usepackage{booktabs}       % professional-quality tables
\usepackage{amsfonts}       % blackboard math symbols
\usepackage{nicefrac}       % compact symbols for 1/2, etc.
\usepackage{microtype}      % microtypography
\usepackage{paralist}

\usepackage[title]{appendix}
\usepackage{amsmath}
\usepackage{amsthm}
\usepackage{amssymb}
\usepackage{bm}
\usepackage{dsfont}

\usepackage{algorithm}
\usepackage{algorithmic}
\usepackage{url,booktabs}
\usepackage{epsfig}
\usepackage{psfrag}
\usepackage{subfigure}
\usepackage{tabulary,multirow}

\usepackage{color}
\usepackage{natbib}
\setcitestyle{open={(},authoryear,close={)}}
\newtheorem{thm}{Theorem}
\newtheorem{lma}{Lemma}
\newtheorem{definition}{Definition}

\renewcommand{\algorithmiccomment}[1]{\bgroup\hfill$\triangleright$~{\scriptsize\textit{#1}}\egroup}

\def \zero {\mathbf{0}}

\def \x {\mathbf{x}}
\def \y {\mathbf{y}}
\def \z {\mathbf{z}}
\def \X {\mathcal{X}}

\def \Z {\mathcal{Z}}

\def \D {\mathcal{D}}
\def \M {\mathbf{M}}
\def \A {\mathbf{A}}

\def \I {\mathbf{I}}
\def \zero {\mathbf{0}}
\def \bb {\mathbf{b}}
\def \R {\mathbb{R}}
\def \N {\mathbb{N}}
\def \bigO {\mathcal{O}}
\def \algo {\mathcal{A}}

\def \E {\mathbf{E}}
\def \TT {\text{T}}
\def \Pr {\text{Pr}}
\def \btheta {\bm{\theta}}

\def \F  {\mathcal{F}}

\def \event {\mathcal{E}}

\begin{document}

\title{Contextual Bandits with Random Projection}

\author{\name Xiaotian Yu \email xtyu@cse.cuhk.edu.hk \\
       \addr Department of Computer Science and Engineering\\
       The Chinese University of Hong Kong\\
       Shatin, N.T., Hong Kong}

%\editor{Kevin Murphy and Bernhard Sch{\"o}lkopf}

\maketitle

\begin{abstract}
Contextual bandits with linear payoffs, which are also known as linear bandits, provide a powerful alternative for solving practical problems of sequential decisions, e.g., online advertisements.  In the era of big data, contextual data usually tend to be high-dimensional, which leads to new challenges for traditional linear bandits mostly designed for the setting of low-dimensional contextual data. Due to the curse of dimensionality, there are two challenges in most of the current bandit algorithms: the first is high time-complexity; and the second is extreme large upper regret bounds with high-dimensional data.  In this paper, in order to attack the above two challenges effectively,  we develop an algorithm of Contextual Bandits via RAndom Projection (\texttt{CBRAP}) in the setting of linear payoffs, which works especially for high-dimensional contextual data.  The proposed \texttt{CBRAP} algorithm is time-efficient and flexible, because it enables players to choose an arm in a low-dimensional space, and relaxes the sparsity assumption of constant number of non-zero components in previous work.  Besides, we provide a linear upper regret bound for the proposed algorithm, which is associated with reduced dimensions. 
\end{abstract}

\section{Introduction}

The Multi-Armed Bandit (MAB) problem was proposed and investigated by Robbins in 1952, which has attracted great interests from numerous researchers in operation research and computer science~\cite{robbins1952some,auer2002finite,bubeck2012regret}.  The fundamental issue in the MAB problem and its variants focuses on the exploration-exploitation trade-off, which refers to an algorithm trying to maximize cumulative rewards in sequential decisions but the algorithm has only limited knowledge about the mechanism of generating the rewards~\cite{auer2002using}.

As a natural and important variant of the basic MAB problem, contextual bandits with linear payoffs, which are also known as linear bandits, are sequential decision-making problems with side information~\cite{wang2005bandit,dani2008stochastic,abbasi2011improved,chu2011contextual}. Specifically, given feature information of arm space for each of $T$ rounds, a learner is required to choose one of $K$ arms.  Linear bandits contain a basic assumption of linearly mapping from the arm space to the reward space~\cite{filippi2010parametric}, which should be the most common case in reality.

Recently, contextual bandits with linear payoffs have been successfully applied into many practical applications. In~\cite{tang2013automatic}, the demonstration of advertisements was based on users' input information on web pages.  The authors formulated the problem of automatic layout selection in online advertisements as a contextual bandit problem.  These personalized advertisements are expected to improve click-through rates of web links.  For these models of sequential decisions, recommendation algorithms always receive additional contextual information from users, which could be greatly useful for the online sequential decisions.

In the big data era, it is pretty common to encounter high-dimensional and/or sparse contextual information.  In this case, traditional bandit algorithms, which are mostly designed in the setting of low-dimensional data, are facing new challenges in applications. Due to the curse of dimensionality, there are two challenges in most the of current bandit algorithms.  The first is high time-complexity; and the second is extremely large upper regret bounds with high-dimensional data.  Specifically,  traditional contextual bandits (e.g.,  \texttt{LinUCB} in~\cite{chu2011contextual}) contain inverse operations in the original contextual space, which will be time-consuming for computations. Besides, the regret bounds of linear bandits in~\cite{chu2011contextual,abbasi2011improved} are related to the original dimension of contextual data, which can lead to increasing regret bounds with the curse of dimensionality. This will be even worse when the original dimension of contextual data is larger than the the total sequential rounds of playing bandits.

There have been some efforts on context bandits with linear payoffs in high-dimensional and/or sparse contextual data~\cite{deshpande2012linear,carpentier2012bandit}.  The corresponding bandit algorithms are named \texttt{BallExp} in~\cite{deshpande2012linear}, and \texttt{SLUCB} in~\cite{carpentier2012bandit}.

However, in \texttt{SLUCB}, the authors assumed that the contextual data contain $S$ non-zero components, which may not be flexible in applications, especially for cases with high-dimensional dense data. In \texttt{BallExp}, the upper and lower regret bounds are relatively loose, and are still closely related to the original dimension of contextual data.  Besides, both \texttt{SLUCB} and \texttt{BallExp} adopted the technique of ball exploration in the high-dimensional space, which will be time-consuming in applications.  For rigorous analysis of time complexity for these two algorithms, interested readers can refer to~\cite{dani2008stochastic}.

Random projection is a powerful and popular technique to deal with high-dimensional data~\cite{fern2003random,zhang2016accelerated}, which maps high-dimensional data onto a low-dimensional space. Note that random projection does not contain the assumption of sparsity in high-dimensional data. The most common case in random projection is to construct a Gaussian random matrix, where each element is an i.i.d. sample following a standard normal distribution. It has been proved to preserve the Euclidean distance within an error ball~\cite{dasgupta1999elementary}.  Besides, the error bounds for inner products in random projection have been investigated~\cite{kaban2015improved}, which will be an effective tool for analysis of upper regret bounds in contextual bandits.

In this paper, to tackle the aforementioned two challenges effectively, we propose an algorithm of Contextual Bandits via RAndom Projection (\texttt{CBRAP}) in the setting of linear payoffs, which works especially for high-dimensional data.  Note that, for simplicity in the work, we assume that contextual bandits have linear payoff functions. But the framework of our bandit algorithm can be easily generalized to the case of relaxing the linear assumption. Specifically, our proposed algorithm adopts random projection to map the high-dimensional contextual information onto a low-dimension space, where we should design a random matrix.  The proposed \texttt{CBRAP} algorithm is time-efficient and flexible, because it enables players to choose an arm in a low-dimensional space, and relaxes the sparsity assumption of constant number of non-zero components in previous work.  Besides, we prove an upper regret bound for the proposed algorithm, and show the bound to be better than the traditional ones with appropriate reduced dimensions. By comparing with three benchmark algorithms (i.e., \texttt{LinUCB}, \texttt{BallExp} and \texttt{SLUCB}), we demonstrate improved performance on cumulative payoffs of the \texttt{CBRAP} algorithm during its sequential decisions on both synthetic and real-world datasets, as well as its superior time-efficiency.

In summary, we make the following contributions.
\begin{compactitem}
  \item For contextual bandits in the setting of linear payoffs, we develop an efficient and practical algorithm named \texttt{CBRAP} by taking advantage of random projection.
  \item We derive an upper regret bound for the proposed \texttt{CBRAP} algorithm, which guarantee the worst case is associated with the reduced dimensions. Besides, our algorithm is more flexible and time-efficient than \texttt{BallExp} and \texttt{SLUCB} in high-dimensional settings.
  \item We evaluate the \texttt{CBRAP} algorithm via a series of experiments with synthetic and real-world datasets.  Compared with the three benchmarks, we demonstrate the proposed algorithm's improved performance of cumulative payoffs during sequential decisions, as well as its time-efficiency.
\end{compactitem}

\section{Preliminary and Related Work}
In this section, we first introduce notions and the definition of sub-Gaussian of a random variable, which will be used in this paper.  Then, we present the process of contextual bandits with linear payoffs, as well as the metric of bandits. Finally, we provide a brief survey on random projection.

\subsection{Notations and Definition of Sub-Gaussian}
The total sequential rounds of playing bandits is $T$.  For each round $t\in [T]$ with $[T] = \{1,2,\cdots,T\}$, a learner receives contextual information from the set of $\X\in \R^n$, where $n$ can be an extremely large integer representing a high-dimensional space.  In this work, high-dimensional contextual data precisely mean that $T\leq n$ or even $T\ll n$, which is the case mentioned in~\cite{carpentier2012bandit}. Let $K\in \N_+$ be the number of arms and $\pi_{t,y}\in [0,1]$ the reward of arm $y$ on round $t$ with $y\in [K]$ and $[K] = \{1,2,\cdots,K\}$.  We adopt $\|\cdot\|_2$ to denote the $\ell^2$ norm of a vector $\x\in\R^n$, and $\I_{m\times m}$ to denote the identity matrix with dimensions of $m\times m$.  For a positive definite matrix $\A\in\R^{m\times m}$, the weighted norm of vector $\x$ is defined as $\|\x\|_{\A}=\sqrt{\x^\TT\A\x}$.  The inner product is represented as $\langle\cdot,\cdot\rangle$, and the weighted inner product is $\x\A^\TT\y = \langle\x,\y\rangle_{\A}$.

Mathematically, we give the following definition on the sub-Gaussian of a random variable.
\begin{definition}[\cite{buldygin1980sub}]
A random variable $\xi$ is sub-Gaussian if there exists an $R\geq 0$ such that
\begin{equation}\label{eqn.sub.Gaussian}
\E[\exp(\lambda \xi)] \leq \exp(\frac{\lambda^2 R^2}{2}),
\end{equation}
where $\lambda\in \R$, $\E[\cdot]$ is the expectation of a random variable and $\exp(\cdot)$ denotes the exponential operation.
\end{definition}

Given a set $\F$, $\xi$ is conditionally $R$-sub-Gaussian if there exists $R\geq 0$, we have $\E[\exp(\lambda \xi)|\F] \leq \exp(\frac{\lambda^2 R^2}{2})$, $\forall \lambda\in \R$.

\subsection{Contextual Bandits with Linear Payoffs}

As shown in~\cite{auer2002using,chu2011contextual}, an algorithm (denoted by $\algo$) for contextual bandits with linear payoffs usually contains the following three steps at round $t$:
\begin{compactenum}
  \item contextual information $\x_{t,y}\in \X$ for all $y\in [K]$ is revealed to the bandit algorithm $\algo$;
  \item the bandit algorithm $\algo$ chooses an arm $a_t \in [K]$, which follows an underlying distribution $\Pi (\x_{t, a_t},\btheta^*)$ with $\btheta^*\in \R^n$ being the unknown true parameter vector; and
  \item a stochastic payoff $\pi_{t,a_t}\in[0,1]$ is revealed to the bandit algorithm $\algo$.
\end{compactenum}

In the above stochastic setting of step 3, for the chosen arm $a_t$ at round $t$, we usually assume that there is an underlying distribution $\Pi (\x_{t, a_t},\btheta^*)$ with the first moment information being $\langle\x_{t,a_t}, \btheta^*\rangle$, so that $\pi_{t,a_t}$ is a sample from $\Pi (\x_{t, a_t},\btheta^*)$.  Thus, we have
\begin{equation}\label{eqn.expect}
\pi_{t,a_t}= \x_{t,a_t}^\TT \btheta^* + \eta_t,
\end{equation}
where $\eta_t$ is a random noise satisfying the assumption of conditionally $R$-sub-Gaussian. That is, $\forall \lambda\in \R$, we have
\begin{equation}\label{eqn.R.sub.Gaussian}
 \E[\exp(\lambda \eta_t)|\F_t]\leq \exp\left(\frac{\lambda^2R^2}{2}\right),
\end{equation}
where $\F_t$ is the $\sigma$-algebra of $\sigma(\{\x_{i,a_i}\}_{i\in[t]},\{\eta_{i}\}_{i\in[t-1]})$ and $R\geq 0$.  Eq.~(\ref{eqn.R.sub.Gaussian}) implies that $\E [\eta_t| \F_t] = 0$.

A popular measure in demonstrating the performance of an algorithm for solving MAB problems is regret, which is defined as the difference between the expected payoff of the optimal decision in hindsight and that of the algorithm.  Mathematically, the regret of the algorithm $\algo$ is defined as
\begin{equation}
Regret(T)\triangleq \E[\sum^T_{t=1} \max_{y\in [K]} \x_{t,y}^\TT \btheta^* - \sum^{T}_{t=1}\pi_{t,a_t}].
\end{equation}

\subsection{Random Projection}
One common technique for dimensionality reduction is to perform linear random projection~\cite{baraniuk2010low,fodor2002survey}. In this paper, we consider projecting the contextual data of $\X\in\R^n$ onto a low-dimensional space of $\Z\in \R^m$.  Without loss of generality, we denote the random projection matrix by $\M \in \R^{m\times n}$. Then, we have
\begin{equation}
\z = \M \x, \label{eqn.random.projection}
\end{equation}
where $\z\in \Z$ and $\x\in \X$.

In~\cite{blum2006random}, $\M$ is constructed as a random matrix where each element follows a normal distribution of $\mathcal{N}\sim (0,\hat{\sigma}^2)$.  By setting $\hat{\sigma}^2 = 1/m$ in the next section, we name our algorithm of \texttt{CBRAP} with Standard Gaussian (SG) matrix (abbreviated as \texttt{CBRAP.SG}).

In~\cite{achlioptas2003database}, the authors proposed new methods for constructing sparse random sign matrix for dimensionality reduction.  In the ensuing section, we name our algorithm of \texttt{CBRAP} with Random Sign (RS) matrix (abbreviated as \texttt{CBRAP.RS}).

In addition to the above work, there have been other ways of constructing a matrix for random projection~\cite{li2006very,ailon2006approximate,clarkson2013low,lu2013faster}.  These investigations consider how to speed up the dimensionality reduction, or how to conduct the random projection with the assumption of low-rank matrix. In this paper, since our focus is to conduct the dimensionality reduction of contextual bandits in a general way, we only consider the construction of random matrix in~\cite{blum2006random,achlioptas2003database}.

\subsection{Related Work}
Contextual bandits are important variants of traditional MAB problems and match many real applications~\cite{langford2008epoch,li2010contextual,tang2013automatic,wang2005bandit}.  Contextual bandits with linear payoffs have been intensively investigated in previous work~\cite{abbasi2012online,abe2003reinforcement,abe1999associative,chu2011contextual,kaelbling1994associative}.  As shown in~\cite{chu2011contextual}, the traditional upper regret bound for the \texttt{LinUCB} algorithm is
\begin{equation}
Regret(T)\leq  \bigO(\sqrt{Tn\ln^3(KT\ln(T)/\delta})),\label{eqn.prior.upper.bound}
\end{equation}
where $\delta\in (0,1)$ is a confidence parameter. We know that the dimension of contextual information of $n$ in Eq.~(\ref{eqn.prior.upper.bound}) will increase when the dimension of context space increases. Roughly, we have the upper regret bound as $R(T)\leq \bigO (T)$ when $n=T$. This will be even worse when the dimension of context data becomes larger, especially for $n\ll T$.

In~\cite{abbasi2012online}, Abbasi-Yadkori et al. studied a sparse variant of stochastic linear bandits.  For high-dimensional bandits, Carpentier and Munos~\cite{carpentier2012bandit} attacked high-dimensional stochastic linear bandits with the sparsity assumption of $S$ non-zero component, where the algorithm is named \texttt{SLUCB}. The upper regret bound in~\cite{carpentier2012bandit} is $\bigO(S\sqrt{T})$.  In real applications, the sparsity assumption may be unreasonable, especially for high-dimensional dense data.

In~\cite{deshpande2012linear}, the authors proposed an algorithm named \texttt{BallExp} for high-dimensional linear bandits, where the regret bound is relatively loose, and is directly related to the dimension of data.

Recently, by adopting additional assumptions of margin and compatibility conditions in~\cite{bastani2015online}, the authors investigated high-dimensional covariates in online decision-marking.

From prior work, it is urgent and important to develop a flexible and practical algorithm for contextual bandits with high-dimensional data, where we do not have additional assumptions (e.g., sparsity or the margin condition).  This motivates our proposed \texttt{CBRAP} algorithm in the next section.

\section{The \texttt{CBRAP} Algorithm}

In this section, we firstly present the overview of \texttt{CBRAP}, and then provide theoretical analyses of a practical upper regret bound and time complexity for the algorithm.

\subsection{Overview of \texttt{CBRAP}}

Our proposed bandit algorithm is shown in Algorithm~\ref{alg.cobra}, which is named \texttt{CBRAP}.  As depicted in Algorithm~\ref{alg.cobra}, the basic idea of \texttt{CBRAP} algorithm is to project the high-dimensional data onto a low-dimensional space, and maintains a confidence set of the unknown optimal parameter $\btheta^*_\z\in\R^m$ in a low-dimensional space. $\btheta^*_\z$ is corresponding to the original true parameter $\btheta^*$ in $n$-dimensional space.
\begin{algorithm}[!t]
   \caption{\texttt{CBRAP}}\label{alg.cobra}
{
\begin{algorithmic}[1]
   \STATE {\bfseries input:} $m$, $T$, $\beta\in \R_+$ and $\alpha\in \R_+$
   \FOR {$p=1,2,\cdots, m$}
   \FOR {$q=1,2,\cdots, n$}
   \STATE generate a random value $b_{pq}$ based on SG or RS
   \STATE $\M(p,q) \leftarrow b_{pq}$
   \ENDFOR
   \ENDFOR
   \STATE $\A_0\leftarrow \I_{m\times m}$
   \STATE $\bb_0\leftarrow \zero_{m}$
   \FOR {$t = 1,2,\cdots, T$}
   \STATE observe context $\x_{t,y}\in\R^{n}$ for all $y\in [K]$
   \STATE $\z_{t,y} \leftarrow \M\x_{t,y}$ for all $y\in [K]$
   \IF {$t == 1$}
   \STATE $\A_t \leftarrow \A_{t-1}$
   \STATE $\bb_t \leftarrow \bb_{t-1}$
   \ELSE
   \STATE $\A_t\leftarrow \A_{t-1} + \z_{t-1,a_{t-1}}\z_{t-1,a_{t-1}}^\TT$
   \STATE $\bb_t\leftarrow \bb_{t-1}+\pi_{t-1, a_{t-1}} \z_{t-1,a_{t-1}}$
   \ENDIF
   \STATE $\btheta^t_\z\leftarrow \A_{t}^{-1}\bb_{t}$
   \FOR {$y\in [K]$}
   \STATE $v_{t,y}\leftarrow \beta\|\z_{t,y}\|_{\A^{-1}_{t}}$
   \STATE $\hat{r}_{t,y}\leftarrow \langle\btheta^{t}_\z,\z_{t,y}\rangle$
   \STATE $\text{ucb}_{t,y} \leftarrow \hat{r}_{t,y} + v_{t,y}$
   \ENDFOR
   \STATE choose the arm $a_t \leftarrow\arg \max_{y\in [K]}~\text{ucb}_{t,y}$ (break ties arbitrarily)
   \STATE observe the reward $\pi_{t, a_{t}}$

   \ENDFOR
\end{algorithmic}
}
\end{algorithm}

Our main contribution in the \texttt{CBRAP} algorithm is two-fold. First, we construct a random matrix for contextual bandits from Step 2 to Step 7 in Algorithm~\ref{alg.cobra}.  Note that the designed random matrix is flexible, and can be revised based on users' needs.  Here we just consider the random matrix in~\cite{blum2006random,achlioptas2003database}. Second, via the designed random matrix $\M$, we conduct dimensionality reduction for the contextual information in Algorithm~\ref{alg.cobra}.
% From Step 20 in Algorithm~\ref{alg.cobra}, we maintain a confidence set of the unknown optimal parameter $\btheta^*_\z\in\R^m$ in a low-dimensional space, which is based on the linear regression in~\cite{abbasi2011improved}.

\section{Theoretical Analysis}
For linear stochastic bandits, we consider the model as $\pi_t = \langle\x_t, \btheta_*\rangle + \eta_t$, where $\x_t\in\R^n$, $\btheta_*\in\R^n$ and $\eta_t$ is conditionally $R$-sub-Gaussian with respect to $\sigma$-algebra as $\F_t = \sigma(\{\x_i\}_{i\in [t]},\{\eta_i\}_{i\in[t-1]})$.

We assume that there exists an approximated feature mapping $\z(\x)\in \R^m$ for each arm $\x\in \R^n$ with $1\leq m \leq n$ and a mapping $\zeta(\btheta_*)\in \R^l$ for the underlying parameter $\btheta_*$ such that for all $\x\in \cup_{t=1}^T \D_t$, $|\z(\x)^\top \zeta(\btheta_*)-\x^\top\btheta_*|\leq \epsilon$, where $\D_t$ is the decision set at time $t$. Here the error $\epsilon>0$ occurs due to dimension reduction.

Let $\hat\mu_t$ be the $\ell^2$-regularized least-squares estimate of $\zeta(\theta_*)$ with regularization parameter $\lambda>0$:
\begin{align}
\hat\mu_t = V_t^{-1}Z_t^\top Y_t,
\end{align}
where $Z_t\in \R^{t\times m}$ is the matrix whose rows are $\z(\x_1)^\top,\cdots,\z(\x_t)^\top$, $Y_t=(\pi_1,\cdots,\pi_t)$ and $V_t = Z_t^\top Z_t+\lambda I_m$. Then we have the following lemma showing that with high probability, $\zeta(\theta_*)$ lies in an ellipsoid centered at $\hat{\mu}_t$.

\begin{lma}\label{lma:rp}
	Assume that for all $\x \in \D$, $\lVert \z(\x)\rVert_2 \leq L$, $\lVert \zeta(\btheta_*)\rVert_2 \leq S$ and $\left| \x^\top \btheta_*\right|\leq B$ and $\z(\x)^\top \zeta(\btheta_*)$ is a $(\gamma,\epsilon)$-approximation of $\x^\top\btheta_*$ on $\D\times \btheta_*$. Let $\beta_t(\delta) = R\sqrt{m\log{\left((1+tL^2)/\delta\right)}} + \sqrt{\lambda}S+\epsilon\sqrt{t}$. Then, for any $\delta > 0$, with probability at least $(1-\delta)(1-\gamma)$, $\zeta(\btheta_*)$ lies in the set $C_t$, where $C_t$ is defined as follows:
	\begin{equation}
	C_t = \{\mu \in \R^m: \lVert \mu-\hat{\mu}_t \rVert_{V_t}\leq \beta_\delta(t)\}.
	\end{equation}
\end{lma}
\begin{proof}
We have
\begin{align}
	&\lVert \hat{\mu}_t - \zeta(\btheta_*)\rVert_{V_t}\nonumber\\
	&=\lVert V_t^{\frac{1}{2}}V_t^{-1}Z_t^\top Y_t - V^{\frac{1}{2}}_tV_t^{-1}V_t \zeta(\btheta_*)\rVert_2\nonumber\\
%	&=\lVert z(X_t)^\top Y_t - V_t \zeta(\theta_*)\rVert_{V^{-1}_t}\nonumber\\
	&=\lVert Z_t^\top \left( Y_t - Z_t\zeta(\btheta_*) \right) - \lambda \zeta(\btheta_*)\rVert_{V_t^{-1}}\nonumber\\
	&\leq\lVert Z_t^\top \left( Y_t - X_t^\top\btheta_*\right)\rVert_{V_t^{-1}}+\lVert \lambda \zeta(\btheta_*)\rVert_{V_t^{-1}}\nonumber\\&+\lVert Z_t^\top \left(X_t^\top\btheta_*- Z_t\zeta(\btheta_*) \right)\rVert_{V_t^{-1}},\label{eqn:conc}
%	=&z^\top V_t^{-1}z(X_t)^\top \left(Y_t- \phi(X_t)\mu_*\right) - \lambda x^TV_t^{-1}z(\mu_*)\\&+ z^\top V_t^{-1}z(X_t)^\top\left(\phi(X_t)\mu_* - z(X_t)z(\mu_*)\right)\\\nonumber
%	\leq& \lVert z \rVert_{V_t^{-1}} \lVert z(X_t)^\top \left(Y_t- \phi(X_t)\mu_*\right)\rVert_{V_t^{-1}} + \lambda \lVert z \rVert_{V_t^{-1}}\lVert z(\mu_*) \rVert_{V_t^{-1}}\\	&+ \lVert z \rVert_{V_t^{-1}} \lVert z(X_t)^\top \left(\phi(X_t)\mu_* - z(X_t)z(\mu_*)\right)\rVert_{V_t^{-1}}.\\
%	\leq& \lVert z \rVert_{V_t^{-1}}\left( R\sqrt{2\log{\frac{det(V_t)^{1/2}det(\lambda I)^{1/2}}{\delta}}} + \sqrt{\lambda\left(S^2+\epsilon\right)} + \epsilon\sqrt{T}\right)
\end{align}
where $X_t = \left(\x_1,\cdots,\x_t\right)$. 

The first term in Eq.~(\ref{eqn:conc}) could be bounded by $R\sqrt{m\log{\left((1+tL^2/\lambda m)\right)+2\log(1/\delta)}} $ with probability $1-\delta$ via directly applying Theorem~2 in ~\cite{abbasi2011improved}. The second term can be easily bounded by $\sqrt{\lambda}S$. For the third term, we can derive that eigenvalues of $Z_tV_t^{-1}Z_t^\top$ are smaller than 1 by conducting Singular-Value Decomposition (SVD) and then we have
\begin{align}
	&\lVert Z_t^\top \left(X_t^\top\btheta_*- Z_t\zeta(\btheta_*) \right)\rVert_{V_t^{-1}}\nonumber\\
	&\leq \lVert X_t^\top\btheta_*- Z_t\zeta(\btheta_*) \rVert_2 \leq \epsilon\sqrt{t},
\end{align}
which completes proof by combining the bounds on the first two terms.
\end{proof}

Based on Lemma~\ref{lma:rp}, we now give a regret upper bound of CBRAP. In CBRAP, we construct a fixed random matrix $\M\in \R^{m\times n}$, $\z(\x) = \M \x$. 
\begin{thm}\label{thm:rg}
Assume the same conditions as shown in Lemma~\ref{lma:rp}. Without loss of generality, let $b_{m,t} =R\sqrt{\log{\left((1+tL^2/\lambda m)\right)}+2\log(1/\delta)}+\sqrt{\lambda}S+B$ and $a_{m,t} = \sqrt{2\log\left( 1+\frac{tL^2}{\lambda m} \right)}$.
%Let $a(l,t) = \sqrt{2\log\left( 1+\frac{t\left(L^2+\epsilon\right)}{\lambda l} \right)}$ and $b(l,t) =\beta_\delta(t)- \epsilon\sqrt{t}=R\sqrt{l\log{\left((1+tL^2)/\delta\right)}} + \sqrt{\lambda\left(S^2+\epsilon\right)}$.
For any $\delta > 0$, with probability at least $(1-2T\exp(-{m\epsilon^2}/{8}))(1-\delta)$, the regret upper bound of CBRAP is
\begin{align}\label{eq.regret}
R_T \leq 2a_{m,T}b_{m,T}m\sqrt{T}+2\left(a_{m,T}\sqrt{m}+1\right)TLS\epsilon_1.
\end{align}
\end{thm}

\begin{proof}
The proof follows~\cite{dani2008stochastic,abbasi2011improved}. Let $\event$ denote the event that for all $\x\in \D$, $\left|\z(\x)^\top \zeta(\btheta_*)-\x^\top\btheta_*\right|\leq \epsilon$. We know that $\event$ holds with probability at least $1-\gamma$. When $\event$ holds and $\zeta(\btheta_*)$ lies in the ellipsoid $C_t$, The instantaneous regret
	\begin{align}
	&r_t = \x_{*,t}^\top\btheta_*-\x_t^\top\btheta_*\nonumber\\
	&\leq \z(\x_{*,t})^\top \zeta(\btheta_*) - \z(\x_t)^\top \zeta(\btheta_*) + 2\epsilon\nonumber\\
	&\leq \z(\x_t)^\top\tilde{\mu}_t - \z(\x_t)^\top \zeta(\btheta_*) + 2\epsilon\nonumber\\
	&\leq \lVert \z(\x_t) \rVert_{V_{t-1}^{-1}} \left(\lVert \tilde{\mu}_t-\hat{\mu}_t +\zeta(\btheta_*)-\hat{\mu}_t \rVert_{V_{t-1}}\right) + 2\epsilon\nonumber\\
	&\leq 2\beta_\delta({t-1}) \lVert \z(\x_t) \rVert_{V_{t-1}^{-1}} + 2\epsilon.
	\end{align}
	Combined with the fact that $r_t \leq 2B$,
\begin{align}
	r_t \leq 2(\beta_\delta({t-1}) + B)\min\{\lVert \z(\x_t) \rVert_{V_{t-1}}, 1\} + 2\epsilon.
\end{align}
Following Lemma~11 in~\cite{abbasi2011improved}, we know that
{\small
\begin{align}
\sum_{t=1}^{T} \min\{\lVert \z(\x_t) \rVert_{V_{t-1}^{-1}}^2, 1\}\leq 2m\log\left( 1+\frac{TL^2}{\lambda m} \right)
\end{align}}
Therefore, the cumulative regret could be bounded by
	{\small
	\begin{align}
	&R_T =\sum_{t=1}^{T}r_t \nonumber\\
	&\leq \sum_{t=1}^{T}2(\beta_\delta({t-1}) +B) \min\{\lVert \z(\x_t) \rVert_{V_{t-1}^{-1}}, 1\} + 2T\epsilon\nonumber\\
	&\leq 2(\beta_\delta(T) +B)\sqrt{T\sum_{t=1}^{T} \min\{\lVert \z(\x_t) \rVert_{V_{t-1}^{-1}}^2, 1\}} + 2T\epsilon \nonumber\\
	&\leq 2(b_{m,T}\sqrt{m} +\epsilon \sqrt{T}) a_{m,T}\sqrt{mT} + 2T\epsilon.
	\end{align}}

Based on~\cite{kaban2015improved}, for $\x \in \R^n$, we have $\Pr\{\left|\x^\top\theta_* - \x^\top \M^\top \M\theta_*\right| > \epsilon_1\lVert \x\rVert_2 \lVert \btheta_* \rVert_2\}<2\exp\left(-\frac{m\epsilon^2_1}{8}\right)$. By taking $\zeta(\btheta_*) = \M\btheta_*$, which is the operation in CBRAP, we directly derive the result. 

When $m=d$, we can set $\zeta(\btheta_*)=\M^{-1}\btheta_*$ and thus have $\epsilon_1=0$, which leads to recovering the regret of linear stochastic bandits, which is $\widetilde{O}(\sqrt{T})$.
\end{proof}

\section{Discusstions}
This is an errata for the theoretical results in {\it CBRAP: Contextual Bandits with RAndom Projection} in AAAI 2017.

\section{Acknowledgments}

I would like to thank Ms. Han Shao for her great support and helpful discussions. 

\vskip 0.2in
\bibliography{bibfile}

\begin{thebibliography}{32}
\providecommand{\natexlab}[1]{#1}
\providecommand{\url}[1]{\texttt{#1}}
\expandafter\ifx\csname urlstyle\endcsname\relax
  \providecommand{\doi}[1]{doi: #1}\else
  \providecommand{\doi}{doi: \begingroup \urlstyle{rm}\Url}\fi

\bibitem[Abbasi-Yadkori et~al.(2011)Abbasi-Yadkori, P{\'a}l, and
  Szepesv{\'a}ri]{abbasi2011improved}
Yasin Abbasi-Yadkori, D{\'a}vid P{\'a}l, and Csaba Szepesv{\'a}ri.
\newblock Improved algorithms for linear stochastic bandits.
\newblock In \emph{NIPS}, pages 2312--2320, 2011.

\bibitem[Abbasi-Yadkori et~al.(2012)Abbasi-Yadkori, Pal, and
  Szepesvari]{abbasi2012online}
Yasin Abbasi-Yadkori, David Pal, and Csaba Szepesvari.
\newblock Online-to-confidence-set conversions and application to sparse
  stochastic bandits.
\newblock In \emph{AISTATS}, pages 1--9, 2012.

\bibitem[Abe and Long(1999)]{abe1999associative}
Naoki Abe and Philip~M Long.
\newblock Associative reinforcement learning using linear probabilistic
  concepts.
\newblock In \emph{ICML}, pages 3--11, 1999.

\bibitem[Abe et~al.(2003)Abe, Biermann, and Long]{abe2003reinforcement}
Naoki Abe, Alan~W Biermann, and Philip~M Long.
\newblock Reinforcement learning with immediate rewards and linear hypotheses.
\newblock \emph{Algorithmica}, 37\penalty0 (4):\penalty0 263--293, 2003.

\bibitem[Achlioptas(2003)]{achlioptas2003database}
Dimitris Achlioptas.
\newblock Database-friendly random projections: Johnson-lindenstrauss with
  binary coins.
\newblock \emph{Journal of computer and System Sciences}, 66\penalty0
  (4):\penalty0 671--687, 2003.

\bibitem[Ailon and Chazelle(2006)]{ailon2006approximate}
Nir Ailon and Bernard Chazelle.
\newblock Approximate nearest neighbors and the fast johnson-lindenstrauss
  transform.
\newblock In \emph{STOC}, pages 557--563, 2006.

\bibitem[Auer(2002)]{auer2002using}
Peter Auer.
\newblock Using confidence bounds for exploitation-exploration trade-offs.
\newblock \emph{Journal of Machine Learning Research}, 3:\penalty0 397--422,
  2002.

\bibitem[Auer et~al.(2002)Auer, Cesa-Bianchi, and Fischer]{auer2002finite}
Peter Auer, Nicolo Cesa-Bianchi, and Paul Fischer.
\newblock Finite-time analysis of the multiarmed bandit problem.
\newblock \emph{Machine learning}, 47\penalty0 (2-3):\penalty0 235--256, 2002.

\bibitem[Baraniuk et~al.(2010)Baraniuk, Cevher, and Wakin]{baraniuk2010low}
Richard~G Baraniuk, Volkan Cevher, and Michael~B Wakin.
\newblock Low-dimensional models for dimensionality reduction and signal
  recovery: A geometric perspective.
\newblock \emph{Proceedings of the IEEE}, 98\penalty0 (6):\penalty0 959--971,
  2010.

\bibitem[Bastani and Bayati(2015)]{bastani2015online}
Hamsa Bastani and Mohsen Bayati.
\newblock Online decision-making with high-dimensional covariates.
\newblock \emph{Available at SSRN 2661896}, 2015.

\bibitem[Blum(2006)]{blum2006random}
Avrim Blum.
\newblock Random projection, margins, kernels, and feature-selection.
\newblock In \emph{Subspace, Latent Structure and Feature Selection}, pages
  52--68. Springer, 2006.

\bibitem[Bubeck and Cesa-Bianchi(2012)]{bubeck2012regret}
S{\'e}bastien Bubeck and Nicolo Cesa-Bianchi.
\newblock Regret analysis of stochastic and nonstochastic multi-armed bandit
  problems.
\newblock \emph{arXiv preprint arXiv:1204.5721}, 2012.

\bibitem[Buldygin and Kozachenko(1980)]{buldygin1980sub}
Valerii~V Buldygin and Yu~V Kozachenko.
\newblock Sub-gaussian random variables.
\newblock \emph{Ukrainian Mathematical Journal}, 32\penalty0 (6):\penalty0
  483--489, 1980.

\bibitem[Carpentier et~al.(2012)Carpentier, Munos,
  et~al.]{carpentier2012bandit}
Alexandra Carpentier, R{\'e}mi Munos, et~al.
\newblock Bandit theory meets compressed sensing for high dimensional
  stochastic linear bandit.
\newblock In \emph{AISTATS}, pages 190--198, 2012.

\bibitem[Chu et~al.(2011)Chu, Li, Reyzin, and Schapire]{chu2011contextual}
Wei Chu, Lihong Li, Lev Reyzin, and Robert~E Schapire.
\newblock Contextual bandits with linear payoff functions.
\newblock In \emph{AISTATS}, pages 208--214, 2011.

\bibitem[Clarkson and Woodruff(2013)]{clarkson2013low}
Kenneth~L Clarkson and David~P Woodruff.
\newblock Low rank approximation and regression in input sparsity time.
\newblock In \emph{STOC}, pages 81--90, 2013.

\bibitem[Dani et~al.(2008)Dani, Hayes, and Kakade]{dani2008stochastic}
Varsha Dani, Thomas~P Hayes, and Sham~M Kakade.
\newblock Stochastic linear optimization under bandit feedback.
\newblock In \emph{COLT}, pages 355--366, 2008.

\bibitem[Dasgupta and Gupta(1999)]{dasgupta1999elementary}
Sanjoy Dasgupta and Anupam Gupta.
\newblock An elementary proof of the johnson-lindenstrauss lemma.
\newblock \emph{International Computer Science Institute, Technical Report},
  pages 99--006, 1999.

\bibitem[Deshpande and Montanari(2012)]{deshpande2012linear}
Yash Deshpande and Andrea Montanari.
\newblock Linear bandits in high dimension and recommendation systems.
\newblock In \emph{CCC}, pages 1750--1754, 2012.

\bibitem[Fern and Brodley(2003)]{fern2003random}
Xiaoli~Zhang Fern and Carla~E Brodley.
\newblock Random projection for high dimensional data clustering: A cluster
  ensemble approach.
\newblock In \emph{ICML}, pages 186--193, 2003.

\bibitem[Filippi et~al.(2010)Filippi, Cappe, Garivier, and
  Szepesv{\'a}ri]{filippi2010parametric}
Sarah Filippi, Olivier Cappe, Aur{\'e}lien Garivier, and Csaba Szepesv{\'a}ri.
\newblock Parametric bandits: The generalized linear case.
\newblock In \emph{NIPS}, pages 586--594, 2010.

\bibitem[Fodor(2002)]{fodor2002survey}
Imola~K Fodor.
\newblock A survey of dimension reduction techniques.
\newblock \emph{Technical Report}, 2002.

\bibitem[Kab{\'a}n(2015)]{kaban2015improved}
Ata Kab{\'a}n.
\newblock Improved bounds on the dot product under random projection and random
  sign projection.
\newblock In \emph{SIGKDD}, pages 487--496, 2015.

\bibitem[Kaelbling(1994)]{kaelbling1994associative}
Leslie~Pack Kaelbling.
\newblock Associative reinforcement learning: A generate and test algorithm.
\newblock \emph{Machine Learning}, 15\penalty0 (3):\penalty0 299--319, 1994.

\bibitem[Langford and Zhang(2008)]{langford2008epoch}
John Langford and Tong Zhang.
\newblock The epoch-greedy algorithm for multi-armed bandits with side
  information.
\newblock In \emph{NIPS}, pages 817--824, 2008.

\bibitem[Li et~al.(2010)Li, Chu, Langford, and Schapire]{li2010contextual}
Lihong Li, Wei Chu, John Langford, and Robert~E Schapire.
\newblock A contextual-bandit approach to personalized news article
  recommendation.
\newblock In \emph{WWW}, pages 661--670, 2010.

\bibitem[Li et~al.(2006)Li, Hastie, and Church]{li2006very}
Ping Li, Trevor~J Hastie, and Kenneth~W Church.
\newblock Very sparse random projections.
\newblock In \emph{SIGKDD}, pages 287--296, 2006.

\bibitem[Lu et~al.(2013)Lu, Dhillon, Foster, and Ungar]{lu2013faster}
Yichao Lu, Paramveer Dhillon, Dean~P Foster, and Lyle Ungar.
\newblock Faster ridge regression via the subsampled randomized hadamard
  transform.
\newblock In \emph{NIPS}, pages 369--377, 2013.

\bibitem[Robbins(1952)]{robbins1952some}
Herbert Robbins.
\newblock Some aspects of the sequential design of experiments.
\newblock \emph{Bulletin of the American Mathematical Society}, 58\penalty0
  (5):\penalty0 527--535, 1952.

\bibitem[Tang et~al.(2013)Tang, Rosales, Singh, and Agarwal]{tang2013automatic}
Liang Tang, Romer Rosales, Ajit Singh, and Deepak Agarwal.
\newblock Automatic ad format selection via contextual bandits.
\newblock In \emph{CIKM}, pages 1587--1594, 2013.

\bibitem[Wang et~al.(2005)Wang, Kulkarni, and Poor]{wang2005bandit}
Chih-Chun Wang, Sanjeev~R Kulkarni, and H~Vincent Poor.
\newblock Bandit problems with side observations.
\newblock \emph{IEEE Transactions on Automatic Control}, 50\penalty0
  (3):\penalty0 338--355, 2005.

\bibitem[Zhang et~al.(2016)Zhang, Zhang, Jin, Cai, and
  He]{zhang2016accelerated}
Weizhong Zhang, Lijun Zhang, Rong Jin, Deng Cai, and Xiaofei He.
\newblock Accelerated sparse linear regression via random projection.
\newblock In \emph{AAAI}, pages 2337--2343, 2016.

\end{thebibliography}

\end{document}